\documentclass{article}

\usepackage{amsfonts,bbm,graphicx,mathtools,multirow,natbib,tikz}
\usepackage{algorithm,algorithmic}
\usepackage{microtype}
\usepackage{subfigure}
\usepackage{makecell}
\usepackage{booktabs} 
\usetikzlibrary{tikzmark}
\include{input}
\usepackage[utf8]{inputenc}

\newsavebox\CBox
\def\textBF#1{\sbox\CBox{#1}\resizebox{\wd\CBox}{\ht\CBox}{\textbf{#1}}}

\usepackage{hyperref}

\usepackage[accepted]{icml2019}

\icmltitlerunning{Random Function Priors for Correlation Modeling}

\begin{document}

\twocolumn[
\icmltitle{Random Function Priors for Correlation Modeling}



\icmlsetsymbol{equal}{*}

\begin{icmlauthorlist}
\icmlauthor{Aonan Zhang}{aonanjohn}
\icmlauthor{John Paisley}{aonanjohn}
\end{icmlauthorlist}

\icmlaffiliation{aonanjohn}{Department of Electrical Engineering \& Data Science Institute,
Columbia University, New York, USA}

\icmlcorrespondingauthor{Aonan Zhang}{az2385@columbia.edu}

\icmlkeywords{Random measures, deep learning}

\vskip 0.3in
]



\printAffiliationsAndNotice{}  

\begin{abstract}
The likelihood model of high dimensional data $X_n$ can often be expressed as $p(X_n|Z_n,\theta)$, where $\theta\mathrel{\mathop:}=(\theta_k)_{k\in[K]}$ is a collection of hidden features shared across objects, indexed by $n$, and $Z_n$ is a non-negative factor loading vector with $K$ entries where $Z_{nk}$ indicates the strength of $\theta_k$ used to express $X_n$. In this paper, we introduce random function priors for $Z_n$ for modeling correlations among its $K$ dimensions $Z_{n1}$ through $Z_{nK}$, which we call \textit{population random measure embedding} (PRME). Our model can be viewed as a generalized paintbox model~\cite{Broderick13} using random functions, and can be learned efficiently with neural networks via amortized variational inference. We derive our Bayesian nonparametric method by applying a representation theorem on separately exchangeable discrete random measures.

\end{abstract}

\section{Introduction}
\label{sec:intro}
Let $X=[X_1,\ldots,X_N]$ be a group of {exchangeable} high dimensional observations, where $X_n\in\bbR^d$. In this paper, we assume $X$ is generated by the model
\begin{align}
p(X) & = \int p(X,Z,\theta)\, dZ d\theta \nonumber\\
& = \int p(Z)p(\theta) \prod_{n\in[N]}p(X_n|Z_n,\theta)\, dZd\theta,
\label{eq:basic_model}
\end{align}
where $p(X_n|Z_n,\theta)$ is a likelihood model conditioned on latent features $\theta\mathrel{\mathop:}=(\theta_k)_{k\in[K]}$ that are shared across the population. $Z_n\mathrel{\mathop:}=[Z_{n1},\ldots,Z_{nK}]$ is a non-negative vector for the $n$th observation, where $Z_{nk}$ determines the extent to which $\theta_k$ is used to express $X_n$. For example, in topic models~\cite{Blei03}, $Z_n$ is a discrete distribution over topics, where $Z_{nk}$ represents the proportion of words in document $n$ sampled from topic $k$. In sparse factor models~\cite{Griffiths11}, $Z_n$ is a binary vector such that latent feature $\theta_k$ contributes to the likelihood if and only if $Z_{nk}=1$. We generically refer to $Z_n$ as a ``non-negative feature loading vector." For exchangeable $X$, it is often assumed the $Z_n$ are exchangeable as well. If we take $Z$ as a feature loading matrix with $Z_n$ as its rows, then $Z$ is \textit{row exchangeable}. By de Finetti's theorem, we can represent
\begin{align}
p(Z)=\int\prod_{n\in\bbN}p(Z_n|\zeta)p(\zeta)d\zeta,
\label{eq:representation_rows}
\end{align}
for some random object $\zeta$. (We let $N=\infty$ in order to apply de Finetti's theorem.) The goal of this paper is to model complex correlations among entries of $Z_n$. Following a common practice, we put an independent prior on $\theta$, $p(\theta)=\prod_{k}p(\theta_k)$ and focus on modeling $p(Z)$.

A straightforward way to model correlation structure is to let $p(Z_n|\zeta)$ be a \textit{parametric} exponential family model. By defining the mean/natural parameters for the model, one can handle correlations to various degrees. For example, $Z_n$ may follow a log-normal distribution~\cite{Lafferty06}, where correlations are modeled through a covariance matrix. However, exponential family models~\cite{Wainwright08} can be rigid, since the number of free parameters is fixed for a certain $K$. To get a more flexible model, it is tempting to consider higher-order moments $\bbE[Z_n^{\otimes M}]$ for a large $M$ up to $K$, where $u^{\otimes M}$ denotes an M-th order outer product of a vector $u$. but in this case the number of free parameters increases exponentially, leading to intractable inference.

In this paper, we use an alternative \textit{Bayesian nonparametric} method to model $Z_n$ as an outcome of random functions, which can handle complex correlations even when $K$ and $M$ go to infinity. Moreover, those random functions can be learned efficiently through inference/decoder networks via amortized variational inference~\cite{Kingma13}. In principle, arbitrarily complex neural networks can be applied to model correlations in our setting.

\begin{figure}
    \centering
    \includegraphics[width=\linewidth]{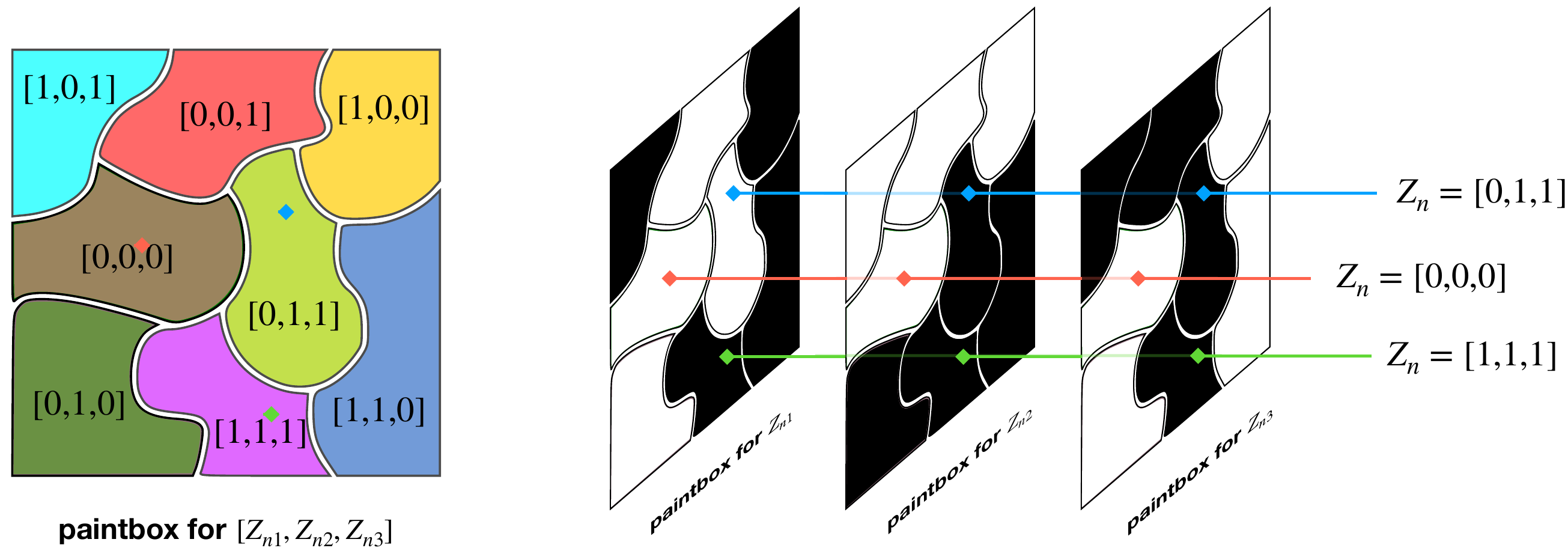}
    \caption{Example of two equivalent representations: (left) Paintbox model for $[Z_{n1},Z_{n2},Z_{n3}]$ by partitioning a unit square into eight regions, one for each distinct value. (right) Factorize the paintbox into three feature paintboxes, one each for a latent feature. Three examples of $u_n$ that determine $Z_n$ are demonstrated as dots in the partition paintbox, and as lines across feature paintboxes.}
    \label{fig:paintbox_toy}
\end{figure}

To give intuition why random function priors are powerful, we first show in Figure \ref{fig:paintbox_toy} an existing feature paintbox model for binary $Z_n$ that illustrates how to model arbitrarily complex correlations using binary random functions~\cite{Broderick13}. For simplicity, let $K=3$. First, select a compact set $S$ in Euclidean space, on which we can define a uniform distribution. For example let $S=[0,1]^2$. Then randomly partition $S$ into eight regions. Each partition represents a possible value for $Z_n=[Z_{n1},Z_{n2},Z_{n3}]$, as shown in Figure \ref{fig:paintbox_toy}. Given the partition, we uniformly sample a point $u_n\sim U(S)$ and assign $Z_n$ be the value defined by the region in which $u_n$ falls. Thus, we translate the problem of modeling distributions on $Z_n$ to modeling \textit{the random partition of $S$}. Following the classic analogy, we call this a partition paintbox model~\cite{Kingman78,Pitman06}. One can further factorize the partition paintbox into ``feature paintboxes"~\cite{Broderick13}. According to Figure \ref{fig:paintbox_toy}, each feature paintbox for the $k$-th feature is randomly partitioned into two regions denoted as $S_k$ (black) and $S_k^c$ (white). Let $Z_{nk}=\mathbbm{1}(u_n\in S_k)$. One can check that the feature paintbox model is the equivalent to the partition paintbox model for arbitrary finite $K$. (Note that here $Z_n$ is a random indicator function.)

The feature paintbox model is redundant but flexible. The arbitrary order moment $\bbE[\prod_{k\in \cJ} Z_{nk}]=\bbE[\text{vol}(\cap_{k\subset \cJ} S_k)]$ for any $\cJ\subset[K]$ can be modeled once we have enough freedom for $S_k$. We summarize the generative process for the feature paintbox model in Algorithm \ref{alg:feature_paintbox} for arbitrary $K$, including $K=\infty$.

We propose a model that can be treated as a generalization of the feature paintbox model from binary to non-negative $Z$ according to a function $Z_{nk}=f_n(\vartheta_k)$. There are two key differences between our model and the feature paintbox model. First, we use data-specific \textit{random functions} $f_n$, instead of points $u_n$, to represent each observation. Second, we use points $\vartheta_k$ from a \textit{Poisson process}, instead of $S_k$, to index each latent feature. A nice property of our model compared to the paintbox model is that we can use deep learning to model $f_n$ through inference and decoder networks~\cite{Kingma13}, allowing for efficient amortized variational inference.

In what follows, Section~\ref{sec:motivation} sets up the problem of modling $Z$ from a random matrix point of view. In Section~\ref{sec:model}, we embed $Z$ as a random measure and derive the functional form of $Z_{nk}=f_n(\vartheta_k)$ through a representation theorem. In Section~\ref{sec:topic_model}, we present a concrete example for Bayesian nonparametric topic modeling together with its amortized variational inference algorithm, and show empirical results in Section~\ref{sec:exp}. Finally, we discuss related work in Section~\ref{sec:discussions} and conclude in Section~\ref{sec:conclusion}.

\begin{algorithm}[t]
\caption{Feature paintboxes model}\label{alg:feature_paintbox}
\begin{algorithmic}[1]
\FOR {$k\in[K]$}
\STATE {Generate a random subset $S_k\subset S$.}
\ENDFOR
\STATE {Guarantee that $\sum_{k\in[K]}\text{vol}(S_k)<\infty$ almost surely.}
\FOR {$n=1,2,\ldots$}
\STATE {Independently generate $u_n\sim U(S)$.}
\STATE {Let $Z_n\!=\![Z_{n1},\ldots,Z_{nK}]$. Set $Z_{nk}\!=\!\mathbbm{1}(u_n\in S_k)$.}
\ENDFOR
\end{algorithmic}
\end{algorithm}

\section{$Z$ as a random matrix?}
\label{sec:motivation}
We will rely on representation theorems to derive the functional form of our models. This usually works out by finding an infinite dimensional random object paired with an exchanegability assumption on that random object. The choice of random objects is the key step, and we will see below that it can be hard to derive an interesting model when choosing a bad random object.

Consider modeling $Z$ as a random matrix. Equation (\ref{eq:representation_rows}) above is one example that derives a mixture representation by assuming row exchangeability of $Z$. However, Equation (\ref{eq:representation_rows}) is uninformative in that, first, it does not tell us what random object $\zeta$ is, and second, it does not determine the connection between $Z_n$ and $\zeta$ through $p(Z_n|\zeta)$. Our discussion in Section~\ref{sec:intro} will show that this provides too much freedom to choose $\zeta$ and $p(Z_n|\zeta)$. 

We further restrict $Z$ by assuming it is \textit{column exchangeable} as well. This requires allowing both $N$ and $K$ to equal infinity. We call $Z$ \textit{separately exchangeable} if it is both row and column exchangeable. Once $K=\infty$, we need to guarantee series convergence for rows. That is, $\sum_{k\in\bbN}Z_{nk}<\infty$ with probability $1$, for any $n\in\bbN$. Row sum convergence is always considered necessary. (For example, in a topic model we want to normalize $Z_n$.) However, the following proposition says that when $Z$ is separately exchangeable, we will get an empty model even for a binary $Z$.

\begin{figure*}[t]
\centering
\includegraphics[width=1\linewidth]{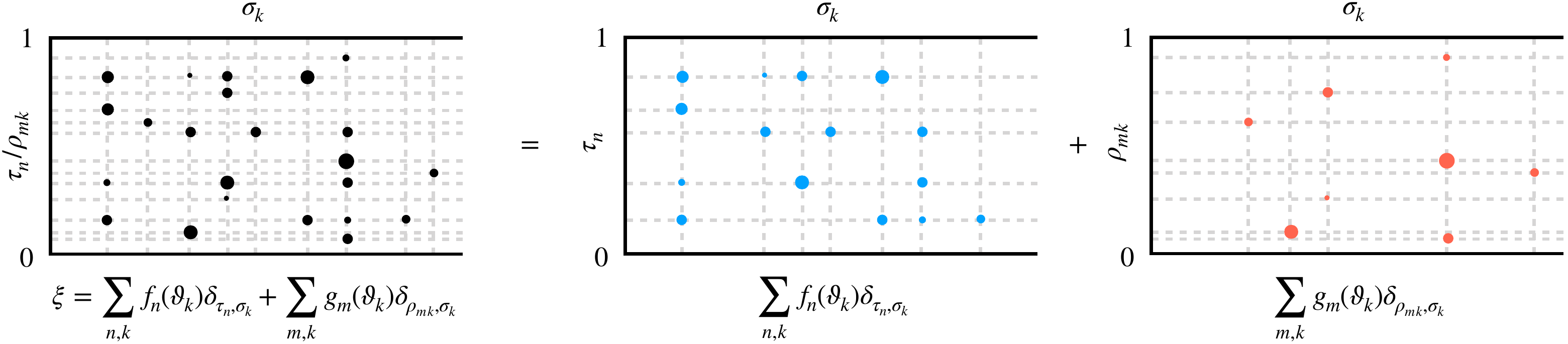}
\caption{Decoupling a separately exchangeable discrete random measure $\xi$ into two parts.}
\label{fig:slfm}
\end{figure*}

\begin{proposition}
An infinite binary matrix $Z$ (i) is separately exchangeable, and (ii) has finite row sums almost surely, if and only if $Z=\bf0$ almost surely.
\label{prop:empty_binary_matrix}
\end{proposition}

\begin{proof}[Proof (Sketch)]
One can prove that $Z$ is a graphon model if it is separately exchangeable~\cite{Hoover79,Aldous85,Orbanz15}. A graphon model satisfies finite row sums if and only if $Z=\bf0$.
\end{proof}

When choosing a bad random object, one can either get a vacuous or an empty model through representation theorems. In the next section, we fix this problem by introducing a nice random object $\xi$ generated by embedding $Z$ as a random measure. Then we apply representation theorems on $\xi$.

\section{$Z$ as a random measure}
\label{sec:model}

\subsection{Population random measure embedding}
In this section, we embed the random matrix $Z$ as a \textit{discrete random measure} $\xi=\sum_{n,k}Z_{nk}\delta_{\tau_n,\sigma_k}$ on an infinite strip $[0,1]\times\bbR_+$, where $(\tau_n)_{n\in\bbN}\subset[0,1]$ distinguishes objects and $(\sigma_k)_{k\in\bbN}\subset\bbR_+$ distinguishes latent features. Both $(\tau_n)_{n\in\bbN}$ and $(\sigma_k)_{_{k\in\bbN}}$ are random as well, and are not necessarily ordered. Note that $\xi$ preserves the matrix structure as demonstrated in Figure \ref{fig:slfm}; the intersection points of horizontal/vertical dashed lines indexed by $(\tau_n)_{n\in\bbN}$ and $(\sigma_k)_{k\in\bbN}$ form an ``equivalent class" of matrix $Z$ up to a re-ordering of rows and columns. The infinite strip is an abstract space introduced solely for applying representation theorems.

Next, we assume $\xi$ is separately exchangeable. That is, $\xi(\cT_1(A)\times\cT_2(B))=_d\xi(A\times B)$ for any measure-preserving transformations $\cT_1$, $\cT_2$ on $[0,1]$ and $\bbR_+$ separately for arbitrary Borel sets $A,B$. Even though the notion of separate exchangeability is different for $\xi$ than for random matrix $Z$, they are conceptually similar, since interchanging row/column indices will not affect the joint distribution. It turns out that we can represent $\xi$ precisely as follows:

\begin{proposition}
A discrete random measure $\xi$ on $[0,1]\times\bbR_+$ is separately exchangeable if and only if
\begin{align}
\xi=\sum_{n,k}f_n(\vartheta_k)\delta_{\tau_n,\sigma_k} + \sum_{m,k}g_m(\vartheta_k)\delta_{\rho_{mk},\sigma_k},
\label{eq:strip_representation}
\end{align}
 almost surely for some random measurable functions $f_n,g_m\geq 0$ on $\bbR_+^2$, a unit rate Poisson process $\{(\vartheta_k,\sigma_k)\}$ on $\bbR_+^2$, and independent $U(0,1)$ arrays $(\tau_n)$ and $(\rho_{mk})$.
\label{prop:representation}
\end{proposition}

\begin{proof}
This follows from the general representation theorem for separately exchangeable random measures on $[0,1]\times\bbR_+$~\cite{Kallenberg06} by removing the non-atomic parts. Details are given in the appendix.
\end{proof}

We briefly look at the two parts of this representation:
\begin{itemize}
\item[1.] $\sum_{n,k}f_n(\vartheta_k)\delta_{\tau_n,\sigma_k}$: This is the part we are interested in. Correlations are learned through coupling of random functions $f_n$ with a Poisson process.
\item[2.] $\sum_{m,k}g_m(\vartheta_k)\delta_{\rho_{mk},\sigma_k}$: This part is less important since the double index in $\rho_{mk}$ means each row (object) slice $\xi(\{\rho_{mk}\},\cdot)$ contains at most one atom. We drop this part in our model.
\end{itemize}
Thus, we can represent $\xi=\sum_{n,k}f_n(\vartheta_k)\delta_{\tau_n,\sigma_k}$ as a coupling of a 2d Poisson process $(\vartheta_k,\sigma_k)$ and random functions $f_n$. As mentioned in Section~\ref{sec:intro}, we derive $Z_{nk}=f_n(\vartheta_k)$. Since we model the entire population through $Z$ by a random measure embedding, we call our model \textit{population random measure embedding} (PRME). 

\subsection{Construction via completely random measures}
\label{subsec:crm}
Once we have a representation for $Z_{nk}$, we still need to guarantee series convergence $\sum_k Z_{nk}=\sum_k f_n(\vartheta_k)<\infty$. This is not obvious, since $\vartheta_k$ spans uniformly on $\bbR_+$. One remedy is to introduce a transformation $\widetilde{\vartheta}_k=T(\vartheta_k)$ that maps almost every $\widetilde{\vartheta}_k$ close to zero, leaving only finite number of $\widetilde{\vartheta}_k$ above any positive threshold.  The method to introduce such a transformation $T$ is via completely random measures (CRM)~\cite{Kingman67}. In the appendix, we show the construction of $T$ via CRMs. In addition, we show that the well-known Indian buffet process~\cite{Ghahramani06,Griffiths11}, its extensions~\cite{Teh09}, hierarchical Dirichlet processes (HDP)~\cite{Teh05} and the discrete infinite logistic normal distribution (DILN)~\cite{Paisley12} are instances of population random measure embeddings. However, these models have restrictions in their model capacity. For example,~\cite{Paisley12} relies on a linear kernel to model correlations and there is no obvious extension to complex kernels. As we will show, a PRME can be more flexible by using nonlinear object-specific functions $f_n$ such as deep neural networks.

\section{An illustration on topic modeling}
\label{sec:topic_model}
\subsection{The model}
In a topic model, we use $Z_n$ to represents an un-normalized discrete distribution over topics, where $Z_{nk}$ is the strength of topic $k$ for document $n$. We use a PRME to model $Z_{nk}$, with the following construction,
\vspace{-3pt}
\begin{align}
&Z_{nk}\sim\text{Gamma}(\beta p_k,\exp(f(h_n,\ell_k))),\nonumber\\
&p_k=V_k\prod_{k'=1}^{k-1}(1-V_{k'}),\quad V_k\sim\text{Beta}(1,\alpha),\nonumber\\
&h_n\sim\cN(0,aI),\quad\ell_k\sim\cN(0,bI),\nonumber\\
&f(h_n,\ell_k)\sim\cN(\mu_f(h_n,\ell_k),\sigma^2_f(h_n,\ell_k)).\vspace{-3pt}
\label{eq:topic_model_z}
\end{align}
We now explain how Equation (\ref{eq:topic_model_z}) relates to the original PRME equation $Z_{nk}=f_n(\vartheta_k)$, via four steps.
\begin{itemize}
\item[1.] \fbox{$f_n(\vartheta_k)~\rightarrow~ f(h_n,\vartheta_k)$}\\
We use a parametric function $f(h_n,\cdot)$ to represent $f_n(\cdot)$, where $f$ is a random function, and $h_n$ is an observation-specific random vector. This decomposition is necessary, since we model $f$ as a normal distribution parameterized by \textit{decoder networks} $\mu_f,\sigma^2_f$, and $h_n$ as the output of an \textit{inference network}.

\item[2.] \fbox{$f(h_n,\vartheta_k)~\rightarrow~f(h_n,\widetilde{\vartheta}_k)$}\\ 
We transform $\widetilde{\vartheta}_k=T(\vartheta_k)$ by transforming the original Poisson process $(\theta_k,\sigma_k)$ to a hierarchical Gamma process~\cite{Teh05,Wang11}. Then we use a stick-breaking construction over $\widetilde{\vartheta}_k$~\cite{Sethuraman94}, where $\widetilde{\vartheta}_k\sim\text{Gamma}(\beta p_k,1)$. $\beta$ is a hyperparameter and $p_k$ is generated by the second line of Equation (\ref{eq:topic_model_z}).

\item[3.] \fbox{$f(h_n,\widetilde{\vartheta}_k)~\rightarrow~f(h_n,\widetilde{\vartheta}_k,\ell_k)$}\\ 
We augment $\widetilde{\vartheta}_k$ to $(\widetilde{\vartheta}_k,\ell_k)$ to introduce extra randomness via $\ell_k$. This operation is equivalent to augmenting the original 2d Poisson process $(\theta_k,\sigma_k)$ to a higher dimensional Poisson process $(\theta_k,\sigma_k,\ell_k)$.

\item[4.] \fbox{$f(h_n,\widetilde{\vartheta}_k,\ell_k)~\rightarrow~\widetilde{\vartheta}_k\cdot\exp(f(h_n,\ell_k))$}\\ We represent $f(h_n,\widetilde{\vartheta}_k,\ell_k)$ as $\widetilde{\vartheta}_k\cdot\exp(f(h_n,\ell_k))$ and assign priors for $h_n$ and $\ell_k$ (line 3 in Equation (\ref{eq:topic_model_z})). We get Equation (\ref{eq:topic_model_z}) by absorbing $\exp(f(h_n,\ell_k))$ into the Gamma scale parameter.
\end{itemize}

In our construction, series convergence $\sum_{k=1}^\infty Z_{nk}<\infty$ can be achieved by bounding $\mu_f$ and $\sigma^2_f$ through a truncation layer in the decoder network.
Given $Z_n$, we sample words in a document, $X_{nm}$ for $m\in [M_n]$, by first sampling its topic assignment $C_{nm}\!\sim\!\text{Disc}(\frac{Z_{n\cdot}}{\sum_k{Z_{nk}}}),$ and then sampling the word from that topic, $X_{nm}\!\sim\!\text{Disc}(\theta_{C_{nm}})$, with topic prior $\theta_k\!\sim\!\text{Dir}(\gamma_0).$ We recall that in topic models, $\theta_k$ (topic $k$) is a discrete distribution over the vocabulary.

\begin{figure*}[ht!]
\centering
\subfigure[]{\includegraphics[width=0.35\linewidth]{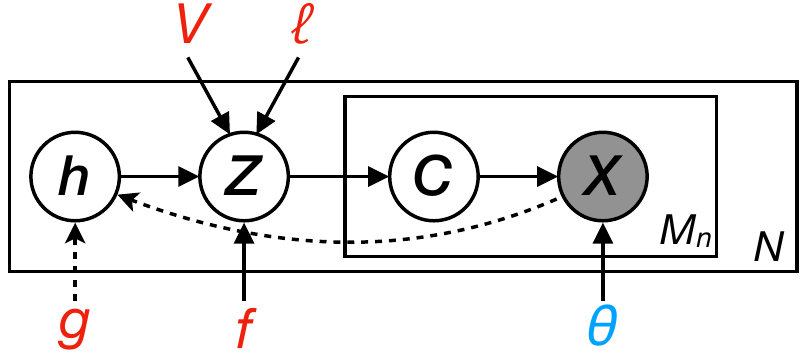}\label{fig:dag}}\qquad
\subfigure[]{\includegraphics[width=.6\linewidth]{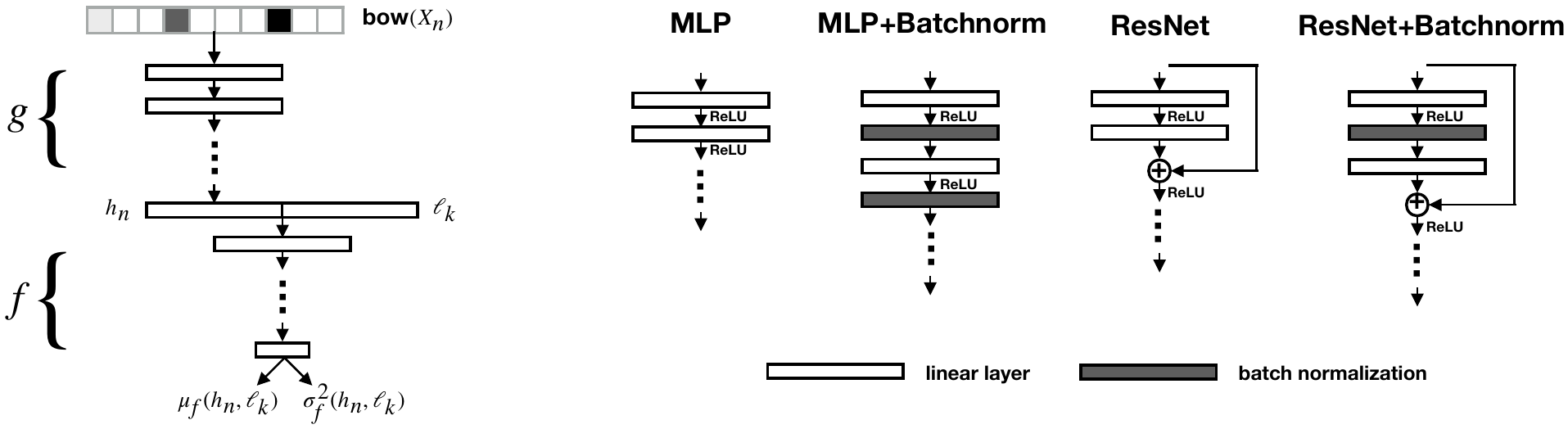}\label{fig:network_architecture_detail}}\vspace{-10pt}
\caption{(a) Graphical representation of our proposed model. Solid arrows represent the generative process and dashed arrows show the VAE part of the posterior. We organize local parameters that belong to a document/word into boxes and remove all sub-indices. We use stochastic natural gradient ascent for $\theta$ and use stochastic gradient ascent for $[\ell,V,g,f]$ (b) Left: The architecture we used in our experiments. Right: Various layer designs.}
\end{figure*}

\subsection{Amortized variational inference}
Assume we have $N$ documents and the posterior is truncated to $K$ topics. The joint likelihood is
\vspace{-8pt}
\begin{align}
&\hspace{-5pt}p(\ell,V,\theta,h,Z,C,X)=\prod_{k=1}^K p(\ell_k)p(V_k)p(\theta_k)\!\prod_{n=1}^N\!\Bigg[ p(h_n)\nonumber\\
&\hspace{-8pt}\prod_{k=1}^K\!p(Z_{nk}|V,h_n,\ell_k)\!\!\prod_{m=1}^{M_n}\!p(C_{nm}|Z_n)p(X_{nm}|C_{nm},\theta)\!\Bigg].\vspace{-8pt}
\end{align}
We use variational inference to approximate the model posterior by optimizing the variational objective function
\begin{align}
\max_q~\cL=\max_q~\bbE_q\Big[\ln\frac{p(\ell,V,\theta,h,Z,C,X)}{q(\ell,V,\theta,h,Z,C)}\Big],
\end{align}
where we restrict $q$ to the factorized family
\vspace{-8pt}
\begin{align}
q(\ell,V,\theta,h,Z,C)=&\prod_{k=1}^K q(\ell_k)q(V_k)q(\theta_k)\prod_{n=1}^N\Bigg[q(h_n|X_n)\nonumber\\
&\prod_{k=1}^K q(Z_{nk})\prod_{m=1}^{M_n}q(C_{nm})\Bigg].
\vspace{-8pt}
\end{align}
Further, for global variables we let
\begin{align}
q(\ell_k)=\delta_{\widehat{\ell}_k},\quad q(V_k)=\delta_{\widehat{V}_k},\quad q(\theta_k)=\text{Dir}(\gamma_k).
\end{align}
For local variables, we introduce an inference network $g$ and let $q(h_n|X_n)=\delta_{g(X_n)}$. For the remaining variables
\begin{align}
&\hspace{-10pt}q(Z_{nk})=\text{Gam}(a_{nk},b_{nk}),~ q(C_{nm})=\text{Disc}(\phi_{nm}).
\end{align}
We use coordinate ascent to update $q$. Each of these updates is guaranteed to improve the objective when the gradient descent step size is small enough~\cite{Nesterov13}. More details are given in the appendix.

For $q(Z_{nk})$, we maximize a lower bound for $\cL$ similar to~\citet{Paisley12}, giving updates
\begin{align}
a_{nk}&=\beta\widehat{p}_k+\sum_{m=1}^{M_n}\phi_{nm}(k),\nonumber\\ b_{nk}&=1/\Big(\bbE\Big[\exp(-f(h_n,\ell_k))\Big]+\frac{M_n}{\varepsilon_n}\Big),
\label{eq:update_z}
\end{align}
where $\varepsilon_n=\sum_{k=1}^K \bbE[Z_{nk}]$.

For $q(C_{nm})$ and $q(\theta)$, we have respective updates
\begin{align}
\phi_{nm}(k)&\propto\exp\Big(\bbE[\ln\theta_{k,X_{nm}}]+\bbE[\ln Z_{nk}]\Big),\label{eq:update_c}\\
\gamma_{kd}&=\gamma_0+\sum_{n=1}^N\sum_{m=1}^{M_n}\phi_{nm}(k)\cdot\mathbbm{1}(X_{nm}=d).
\label{eq:update_eta}
\end{align}


For $[\ell,V,g,f]$, we do gradient ascent on $\cL$. Batch variational inference can be done via coordinate ascent by iteratively updating the above variables. Dependencies among variables are shown in Figure \ref{fig:dag}.


For stochastic inference, in each global iteration we sample a subset $N_t\subset[N]$ and compute the noisy variational objective
\begin{align}
\cL_t&\!=\!\bbE\Big[\ln p(\ell,V,\theta)\Big]\!+\!\frac{N}{|N_t|}\sum_{n\in N_t}\bbE\Big[\ln p(h_n,Z_n,C_n,X_n)\Big]\nonumber\\ &+\bbH\Big[q(\theta)\Big]\!+\!\frac{N}{|N_t|}\sum_{n\in N_t}\bbH\Big[q(Z_n,C_n)\Big].
\end{align}

Optimizing local variables $Z,C$ can be done via closed-form updates exactly as in the batch case. For the other parameters we use stochastic gradient methods. Let $\rho^{(t)}\propto (t_0+t)^{-\kappa}$ be the step size with some constant $t_0$ and $\kappa\in(0.5,1]$. We apply the stochastic natural gradient method~\cite{Hoffman13} for $\theta$
\begin{align}
\widetilde{\gamma}_{kd}^{(t)}&=\gamma_0+\sum_{n\in N_t}\sum_{m=1}^{M_n}\phi_{nm}(k)\cdot\mathbbm{1}(X_{nm}=d),\nonumber\\ \gamma_{kd}^{(t)}&=(1-\rho^{(t)})\gamma_{kd}^{(t-1)}+\rho^{(t)}\widetilde{\gamma}_{kd}^{(t)}.
\label{eq:stochastic_update_eta}
\end{align}
and stochastic gradient method for the rest,
\begin{align}
[\ell,V,g,f]^{(t)} = [\ell,V,g,f]^{(t-1)}\!+\!\rho^{(t)}\nabla_{[\ell,V,g,f]}\cL_t.
\label{eq:stochastic_update_rest}
\end{align}
Since in each iteration we only do one gradient step, the cost is low. Note that through the variational autoencoder (VAE)~\cite{Kingma13} we transfer local updates for $h_n$ to global update for $g$, which will significantly speed-up inference. We summarize the stochastic inference algorithm in Algorithm \ref{alg:stochastic_inference}.

\begin{algorithm}[t]
\caption{Stochastic inference algorithm}\label{alg:stochastic_inference}
\begin{algorithmic}[1]
\FOR {$t=1,2,\ldots$}
\STATE {Sample a subset $N_t\subset[N]$}
\STATE {\textbf{\underline{Update local variables}}}
\WHILE {not converge}
\STATE {Closed-form update $q(Z_n)$ for $n\in N_t$. ~ Eq.~(\ref{eq:update_z})}
\STATE {Closed-form update $q(C_n)$ for $n\in N_t$. \, Eq.~(\ref{eq:update_c})}
\ENDWHILE
\STATE {\textbf{\underline{Update global variables}}}
\STATE {Noisy natural gradient step for $q(\theta)$.\hfill Eq.~(\ref{eq:stochastic_update_eta})}
\STATE {Noisy gradient step for $\ell,V,g,f$. ~~~~~~~~~~~~~~ Eq.~(\ref{eq:stochastic_update_rest})}
\ENDFOR
\end{algorithmic}
\end{algorithm}

\begin{table*}[t]
\centering 
\caption{Perplexity result for text data sets with different dictionary sparsity levels controlled by $\gamma_0$.}\vspace{3pt}
\label{tab:batch_text_various_gamma}
\resizebox{\linewidth}{!}{
  \begin{tabular}{ |l|c c c c|c c c c|c c c c| }
    \hline
    \multirow{2}{*}{\textbf{Model}} & \multicolumn{4}{c|}{\textbf{New York Times}} & \multicolumn{4}{c|}{\textbf{20Newsgroups}} & \multicolumn{4}{c|}{\textbf{NeurIPS}}\\
    & $\gamma_0\!=\!0.2$ & $\gamma_0\!=\!0.4$ & $\gamma_0\!=\!0.6$ & $\gamma_0\!=\!0.8$ & $\gamma_0\!=\!0.2$ & $\gamma_0\!=\!0.4$ & $\gamma_0\!=\!0.6$ & $\gamma_0\!=\!0.8$ & $\gamma_0\!=\!0.2$ & $\gamma_0\!=\!0.4$ & $\gamma_0\!=\!0.6$ & $\gamma_0\!=\!0.8$\\\hline
    HDP & 2436.51 & 2464.74 & 2482.61 & 2501.82 & 5317.68 & 5845.90 & 6294.68 & 6665.68 & 1973.39 & 1962.90 & 1981.83 & 2009.58 \\
    DILN & 2231.16 & 2295.12 & 2418.16 & 2509.24 & 5164.93 & 5732.12 & 6143.64 & 6389.99 & 1853.89 & 1902.88 & 1944.90 & \textBF{1947.94} \\
    PRME & \textBF{2203.00} & \textBF{2247.25} & \textBF{2299.60} & \textBF{2338.38} & \textBF{5102.08} & \textBF{5531.04} & \textBF{5878.39} & \textBF{5975.12} & \textBF{1753.61} & \textBF{1850.37} & \textBF{1917.21} & 1953.85 \\\hline
  \end{tabular}}\vspace{-10pt}
\end{table*}

\subsection{Network architectures}
\label{subsec:network_architecture}
The flexibility of our model comes from the inference and decoder networks $g$ and $f$. As we show in the experiments, these allow us to learn complex non-linear ``paintboxes" in order to capture complex topic correlations. Since optimizing over deep neural networks is still a challenging problem in theory, we design our networks with architectures that work well in practice. Rather than directly applying multilayer perceptrons~\cite{Rumelhart85}, we instead use more complex layer designs such as batch normalization~\cite{Ioffe15} and deep residual networks (ResNet)~\cite{He16} to speed-up training. For inference network $g$, we use the bag-of-words representation of $X_n$ as the input feature. For decoder network $f$, we concatenate $h_n=g(X_n)$ and $\ell_k$ as inputs. Detailed architecture design is shown in Figure \ref{fig:network_architecture_detail}.

\section{Experiments}
\label{sec:exp}

\begin{table}[t]
\centering\vspace{-8pt}
\caption{Dataset description.}\vspace{3pt}
\label{tab:datasets}
\resizebox{.9\linewidth}{!}{
  \begin{tabular}{ |c|c c c c| }
    \hline
    \textbf{Corpus} & \# train & \# test & \# vocab & \# tokens\\\hline
    New York Times & 5,000 & 500 & 8,000 & 1.4M\\
    20Newsgroups & 11,269 & 7,505 & 53,975 & 2.2M\\
    NeurIPS & 2,183 & 300 & 14,086 & 3.3M\\\hline
  \end{tabular}}
\end{table}

\begin{table}[h!]
\centering\vspace{-8pt}
\caption{Network layer configurations for New York Times dataset.}\vspace{3pt}
\label{tab:network_layer_detail}
\resizebox{.9\linewidth}{!}{
  \begin{tabular}{ |c|c c| }
    \hline
    \textbf{Depth} & \textbf{Inference Network} & \textbf{Decoder Network}\\\hline
    \multirow{2}{*}{2 layers} & \multirow{2}{*}{$[8000\times d_h]$} & \multirow{2}{*}{\makecell{$[(d_h+d_\ell)\times80]$\\$[80\times2]$}} \\
    & & \\\hline
    \multirow{3}{*}{4 layers} & \multirow{3}{*}{\makecell{$[8000\times1000]$\\$[1000\times d_h]$}} & \multirow{3}{*}{\makecell{$[(d_h+d_\ell)\times80]$\\$[80\times80]$\\$[80\times2]$}} \\
    & & \\
    & & \\\hline
    \multirow{4}{*}{6 layers} & \multirow{4}{*}{\makecell{$[8000\times1000]$\\$[1000\times1000]$\\$[1000\times d_h]$}} & \multirow{4}{*}{\makecell{$[(d_h+d_\ell)\times80]$\\$[80\times80]$\\$[80\times80]$\\$[80\times2]$}} \\
    & & \\
    & & \\
    & & \\\hline
    \multirow{5}{*}{8 layers} & \multirow{5}{*}{\makecell{$[8000\times1000]$\\$[1000\times1000]$\\$[1000\times1000]$\\$[1000\times d_h]$}} & \multirow{5}{*}{\makecell{$[(d_h+d_\ell)\times80]$\\$[80\times80]$\\$[80\times80]$\\$[80\times80]$\\$[80\times2]$}} \\
    & & \\
    & & \\
    & & \\
    & & \\\hline
  \end{tabular}}
\end{table}

\begin{table}[t]
\centering\vspace{-8pt}
\caption{Perplexity result for various network depths.}\vspace{3pt}
\label{tab:network_architecture_depth}
\resizebox{1\linewidth}{!}{
  \begin{tabular}{ |c|c c c c| }
    \hline
    \textbf{Depth} & \textbf{MLP} & \textbf{MLP+BN} & \textbf{ResNet} & \textbf{ResNet+BN}\\\hline
    2 layers & 2325.84 & 2327.81 & N/A & N/A\\
    4 layers & 2228.62 & 2203.00 & 2214.02 & 2195.72\\
    6 layers & 2219.06 & \textBF{2184.44} & 2202.79 & 2194.74\\
    8 layers & \textBF{2196.35} & 2195.68 & \textBF{2199.07} & \textBF{2184.56}\\\hline
  \end{tabular}}\vspace{-10pt}
\end{table}

\begin{table}[t!]
\centering\vspace{-8pt}
\caption{Perplexity result for various size of $h_n$/$\ell_k$.}\vspace{3pt}
\label{tab:network_architecture_width}
\resizebox{1\linewidth}{!}{
  \begin{tabular}{ |c|c c c c| }
    \hline
    \textbf{Hidden Size} & \textbf{MLP} & \textbf{MLP+BN} & \textbf{ResNet} & \textbf{ResNet+BN}\\\hline
    $d_h\!=\!d_\ell\!=\!2$ & 2287.40 & 2258.97 & 2265.53 & 2256.84\\
    $d_h\!=\!d_\ell\!=\!5$ & 2245.43 & 2243.26 &2231.54 & 2225.64\\
    $d_h\!=\!d_\ell\!=\!10$ & \textBF{2220.82} & 2217.65 & 2227.04 & 2199.73\\
    $d_h\!=\!d_\ell\!=\!20$ & 2228.62 & \textBF{2203.00} & \textBF{2214.02} & \textBF{2195.72}\\\hline
  \end{tabular}}
\end{table}

\subsection{Batch experiments}
We show empirical results on three text datasets: a 5K subset of New York Times, 20Newsgroups, and NeurIPS. Their basic statistics are shown in Table~\ref{tab:datasets}. For each test document $X_n$, we do a $90\%/10\%$ split into training words $X_{n,TR}$ and testing words $X_{n,TS}$. The perplexity is calculated based on the prediction of $X_{n,TS}$ given the model and $X_{n,TR}$,
\begin{align}
\textstyle \text{perplexity}\!=\!\exp\!\Big(\!\!-\frac{\sum_{m\in X_{n,TS}}\ln p(X_{nm}|X_{n,TR})}{|X_{n,TS}|}\Big).
\label{eq:perplexity}
\end{align}
Lower perplexity means better predictive performance.

In Table~\ref{tab:batch_text_various_gamma}, we compare three Bayesian nonparametric models: hierarchical Dirichlet process (HDP)~\cite{Teh05}, discrete infinite logistic normal (DILN)~\cite{Paisley12}, and our population random measure embedding (PRME) using 4-layer MLP with batch normalization.\footnote{The number of layers includes inference network and decoder network. We ignore the last layer of the decoder network.} 
We tune $\gamma_0$ and fix the truncation level $K=100$ and set the $a=1,b=1,\alpha=1,\beta=5$ for fair comparisons. All gradient updates are done via Adam~\cite{Kingma14} with learning rate $10^{-4}$. As Table~\ref{tab:batch_text_various_gamma} shows, PRME consistently perform better than HDP and DILN. Where DILN was designed to outperform HDP by learning topic correlation structure, PRME improves upon DILN by learning a more complex kernel structure.

\begin{figure*}[t!]
\includegraphics[width=1\linewidth]{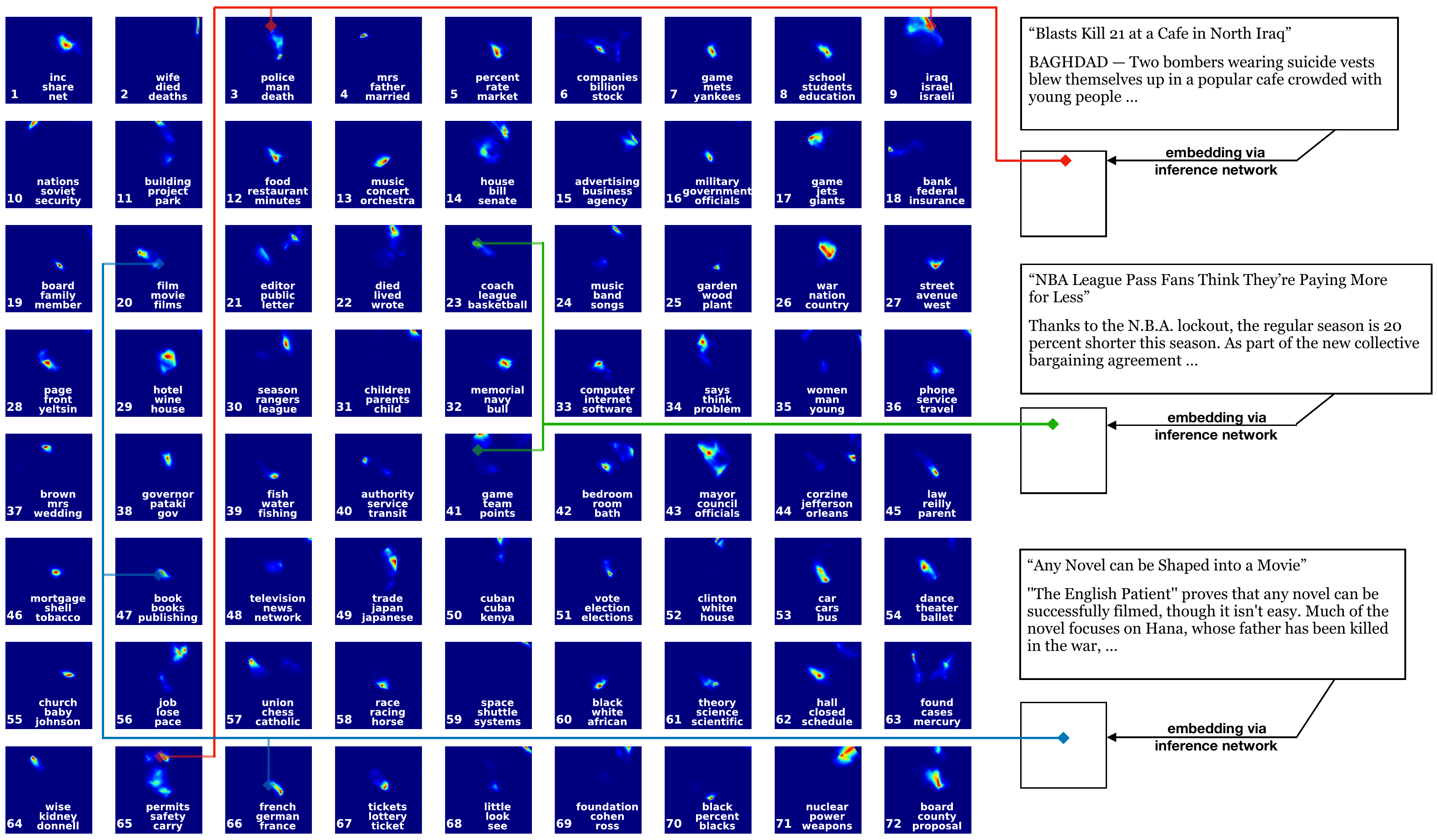}
\caption{A paintbox demonstration of salient topics learned from the one million New York Times dataset. In each paintbox on the LHS, pixel $(x,y)$ represents the topic strength $Z_{(x,y),k}$ as a function of $h_{(x,y)}$ for a particular topic $k$. We also show embeddings of three articles in the same space, as well as their projection onto selected paintboxes. Each article is connected to its most-used topics.}
\label{fig:topic_paintbox}
\end{figure*}

Since PRME encodes complex correlation patterns with a neural network, we further consider the influence of network architecture on perplexity for the New York Time dataset. We compare four layer designs: multilayer perceptron (MLP), MLP with batch normalization (MLP+BN), ResNet, and ResNet with batch normalization (ResNet+BN); see Figure \ref{fig:network_architecture_detail} for details. In Table~\ref{tab:network_architecture_depth} and Table~\ref{tab:network_architecture_width}, we separately tune the depth of each network and the hidden size of $h/\ell$ while holding other parameters fixed. The details of layer sizes can be found in Table~\ref{tab:network_layer_detail}. We observe that the perplexity result tend to be better when we scale up the network depth/width. Batch normalization and ResNet both improve performance.

\subsection{Online experiments}
For the larger one million New York Times dataset, we show ``topic paintboxes" learned with stochastic PRME in Figure \ref{fig:topic_paintbox}.\footnote{We set $t_0=100,\kappa=0.75$ and use a 6-layer MLP.}
In Figure \ref{fig:topic_paintbox}, each paintbox corresponds to one topic whose top words are displayed inside the box. The color of a pixel $(x,y)$ in the $k$-th paintbox ranges from blue (small value) to red (large value) and represents mean topic strength $\bbE[Z_{(x,y),k}]=\beta\widehat{p}_k\bbE[\exp(f(h_{(x,y)},\ell_k))]$ as a function of $h_{(x,y)}$ for topic $k$. To define $h_{(x,y)}$ for 2d visualization, we collect the empirical embeddings $H=[h_1,\ldots,h_N]^\top=[g(X_1),\ldots,g(X_N)]^\top$ on a subset of data, subtract their mean $m_h$, and use the SVD to select the two most informative directions $\widetilde{h}_1,\widetilde{h}_2$ with singular values $s_1,s_2$. Then we plot each paintbox as the function value $\bbE[Z_{(x,y),k}]=\beta\widehat{p}_k\bbE[\exp(f(m_h+x s_1\widetilde{h}_1+y s_2\widetilde{h}_2,\ell_k))]$ by tuning $(x,y)\in[-0.2,0.2]^2$. 

The correlation between topics can be read out from the paintboxes. Those paintboxes that have overlapping salient regions tend to be more correlated. For example, topic 13 [music, concert, orchestra], topic 20 [film, movie, films], and topic 47 [book, books, publishing] share a salient region, which gives a third-order positive correlations over those topics. In principle, the paintbox can explain arbitrary order correlations as the neural network complexity increases. We observe that each paintbox in Figure \ref{fig:topic_paintbox} consists of multiple contiguous salient regions. This is due to the smoothness of neural networks, since $g(X_{n_1})\approx g(X_{n_2})$ when $X_{n_1}$ and $X_{n_2}$ share similar words. Also, the various ``modes" in each paintbox demonstrate the greater flexibility of neural networks in explaining different contexts of a topic.

In Figure \ref{fig:topic_paintbox}, we also display three documents with their embeddings $h_n$ projected onto the 2d paintbox space. Each embedding hits salient regions of several paintboxes. Thus, each document can be interpreted as a mixture of these corresponding topics. We again note that we only display the paintbox in 2d via post-processing, but the actual paintbox is in 20 dimension; a higher-dimensional paintbox can be more complex than what is shown.

We can compare the difference between paintboxes for PRME in Figure \ref{fig:topic_paintbox} and paintboxes for binary random measures in Figure \ref{fig:paintbox_toy}. First, the paintbox for PRME is real-valued, so it is natural to use smooth functions to model it. In the binary case the paintbox is zero/one valued; in this case one can apply a threshold function over the PRME paintbox to binarize it. Second, in contrast to the binary paintbox, each PRME paintbox is unbounded. We control the area of this salient region through regularization.

Figure \ref{fig:online_exp}(a) demonstrates the perplexity of DILN and PRME with various decay speed $\kappa$ on a held-out test set of size 3K. PRME converges after seeing one million documents, and it performs better than DILN. Also, online learning is much more efficient than batch learning with various training data size, as shown in Figure \ref{fig:online_exp}(b). In Figure \ref{fig:online_exp}(c), we compare run times for updating local parameters ($[Z,C]$ for PRME) and global parameters ($[\theta,\ell,V,g,f]$ for PRME) with batch size 500. Since the cost is very imbalanced between local and global, for demonstration purpose we compare the cost between five local iterations and one global iteration. In our experiments, local updates requires around 20 iterations to converge. Compared with DILN, PRME costs much less in local and costs more in global updates, since it uses the VAE to transfer local updates for $h_n$ into global updates for $g$. The extra global cost ($\sim$0.35s) is significantly smaller than the reduced local cost ($\sim$4s), even when using a deep network architecture. Finally, Figure \ref{fig:online_exp}(d) demonstrates the usage proportion for all topics. PRME tends to use a subset of the 100 available topics in the truncated posterior, indicating use by the model of this nonparamteric feature.

\begin{figure*}[t!]
\centering
\includegraphics[width=1\linewidth]{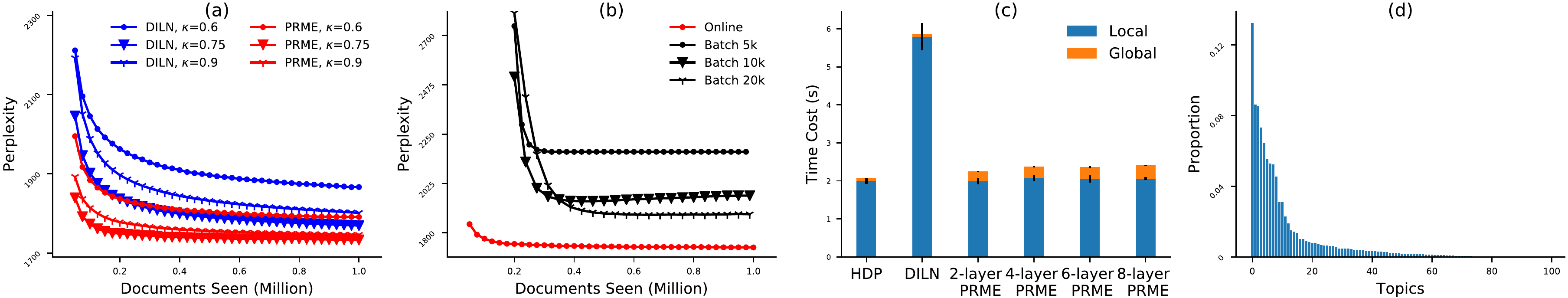}
\caption{(a) Online performance comparisons between DILN and PRME. (b) Online versus batch. (c) Time cost comparison between updating local and global variables. (d) Ranked topic usage proportions in the posterior, indicating nonparametric functionality.}
\label{fig:online_exp}
\end{figure*}

\section{Discussion}
\label{sec:discussions}
\subsection{Connections with other random objects}

Another view is to treat $(Z_{nk})_{n\in[N],k\in[K]}$ as a bipartite graph over objects $[N]$ and atoms (features) $[K]$ with edge strength $Z_{nk}$. An important topic in random graph theory is to study the total strength of edges $|E|=\sum_{n\in[N],k\in[K]}\bbE[Z_{nk}]$ asymptotically as a function of $N$. There has been extensive work on random graphs, networks, and relational models~\cite{Roy08,Miller09,Caron12,Lloyd12,Veitch15,Cai16,Lee16,Crane17,Caron17,Caron17b}, but these methods mainly focus on dense graphs where $|E|\sim\cO(N^2)$, and sparse graphs where $|E|\sim\cO(N^{1+\alpha})$ with $0<\alpha<1$ or $|E|\sim\cO(N\log N)$. Our method offers a new solution to \textit{extremely sparse hidden graphs} where $|E|\sim\cO(N)$, by coupling random functions and a Poisson process. Our solution cannot be trivially derived from previous representations in sparse/dense graphs.  There is a developed probability theory building connections between exchangeable binary random measures and functions on combinatorial structures among atoms~\cite{Pitman95,Pitman06,Broderick13,Broderick15,Heaukulani16,Campbell18}.

Our topic model construction is motivated by previous research on dependent random measures~\cite{Zhou11,Paisley12,Chen13,Foti13,Zhang15,Zhang16}. Our focus is to place \textit{mild exchangeability assumptions} on a population random measure $\xi$ and derive a very general random function model through representation theorems. Hence our use of neural networks to achieve this task. We mention that our method can also be adapted to non-exchangeable settings.

\subsection{Deep hierarchical Bayesian models}\vspace{-3pt}
One can scale up model capacity by stacking multiple one-layer Bayesian nonparametric models such as Dirichlet processes~\cite{Teh05}, beta processes~\cite{Thibaux07}, and Gamma processes~\cite{Zhou15,Zhou18}. Population random measure embedding uses a different strategy by constructing random measures as a coupling of random functions with {a single Poisson process}. In this way, we transfer all the model complexity into random functions $f_n$. Using amortized variational inference, we transfer posterior inference of discrete random measures into optimizing neural networks, which is much more efficient.

\subsection{Posterior inference bottleneck}\vspace{-3pt}
Efficient posterior inference is essential in Bayesian nonparametric methods where conjugacy often does not hold~\cite{Broderick14,Zhang16b}. In principle, one can apply a simple prior on $Z$ and still rely on accurate posterior inference to resolve the structure. However, posterior inference for random measures is not simple because complex correlations among atoms leads to slow MCMC mixing. Instead, one can approximate the posterior using variational methods~\cite{Blei17} and try to learn a $q$ distribution with good approximation quality~\cite{Paisley12b,Hoffman15,Ranganath16,Tran17}. Our method introduced a structured prior to regularize variational inference. Empirical results showed that we get an interpretable posterior.

\section{Conclusion and Future Work}
\label{sec:conclusion}
We presented random function priors to handle complex correlations among features via a population random measure embedding. We further derived a new Bayesian nonparametric topic model to demonstrate the effectiveness of our method for learning topic correlations through deep neural networks with amortized variational posterior inference. In future work, we will consider the more challenging task of removing the non-differentiable Poisson process and making our model fully differentiable. 

\section*{Acknowledgements}
We thank Howard Karloff and Victor Veitch for their helpful comments during the early stage of this work. This research was supported in part by funding from Capital One Labs in New York City.

\bibliography{sample}
\bibliographystyle{icml2019}


\clearpage
\newpage

\section*{Appendix}
\label{sec:appendix}
\begin{figure*}[t]
\centering
\includegraphics[width=.8\linewidth]{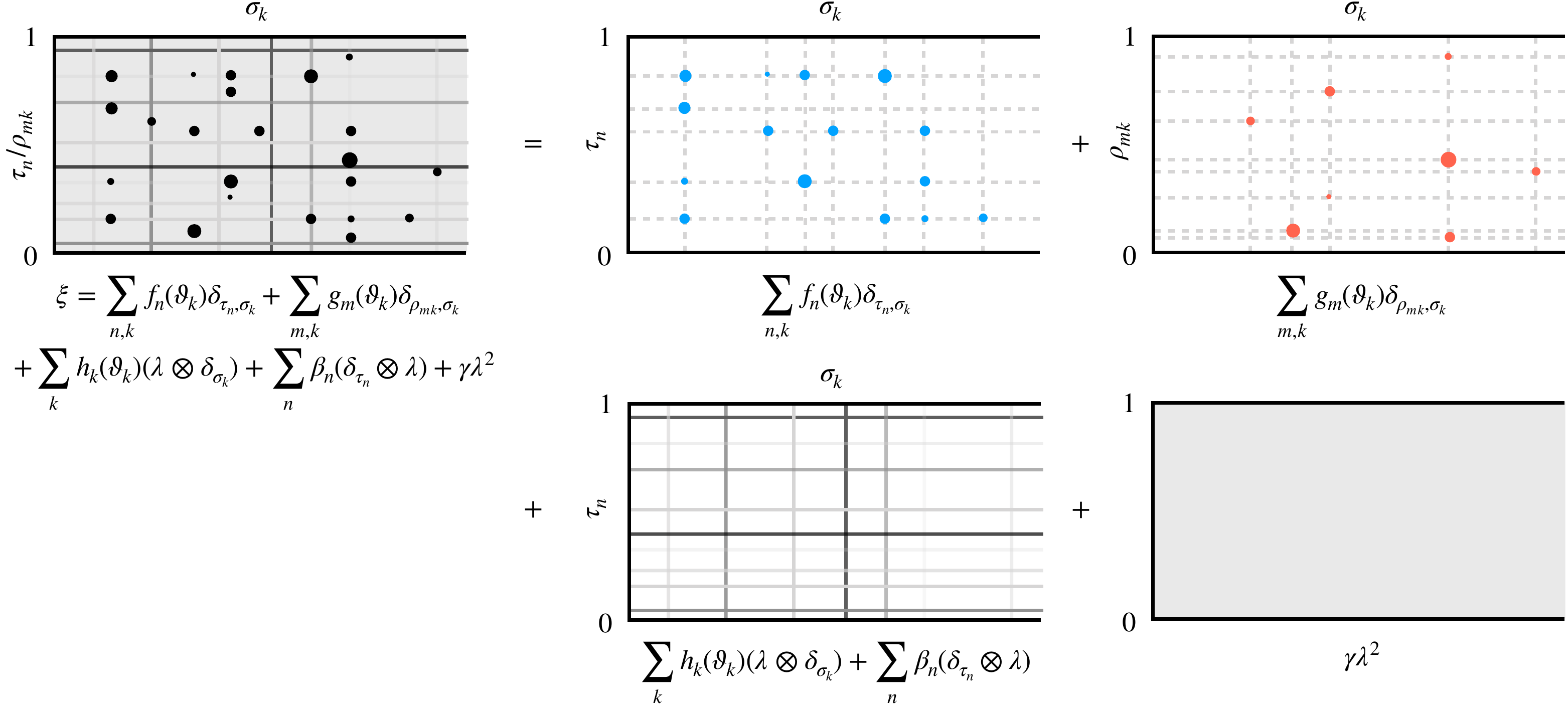}\vspace{-10pt}
\caption{Decouple separately exchangeable random measure $\xi$ into four parts.}
\label{fig:slfm_appendix}
\end{figure*}

\subsection*{Proof of Proposition~\ref{prop:empty_binary_matrix}.}
\begin{proof}
From ~\cite{Aldous85,Hoover79,Orbanz15}, we can represent every separately exchangeable infinite binary matrix $Z=(Z_{nk})$ if and only if it can be represented as follows: There is a random function $W:[0,1]^2\rightarrow[0,1]$ such that
\begin{align}
(Z_{nk})\stackrel{d}{=}(\mathbbm{1}(R_{nk}<W(U_n,V_k))).
\end{align}
Thus, one can reconstruct $Z$ by first sample $W,(U_n),(V_k)$ and then sample $Z_{nk}|~W,(U_n),(V_k)\sim\text{Bernoulli}(W_{U_nV_k})$ through independent coin flips. It is straightforward to prove that the finite row sum assumption can only be satisfied when $\int_{[0,1]^2} W(u,v)dudv=0$. When that happens, $Z=\bf0$ almost surely.
\end{proof}

\subsection*{Proof of Proposition~\ref{prop:representation}.}
\begin{proof}
The representation theorem in Proposition~\ref{prop:representation} is immediate from a more general result of separately exchangeable random measures. We temporarily reload notations $h_k,\beta_n$.

\begin{citedthm}[\cite{Kallenberg06}]
A random measure $\xi$ on $[0,1]\times\bbR_+$ is separately exchangeable if and only if almost surely
\begin{align}
\xi&=\underbrace{\sum_{n,k}f_n(\vartheta_k)\delta_{\tau_n,\sigma_k} + \sum_{m,k}g_m(\vartheta_k)\delta_{\rho_{mk},\sigma_k}}_{\text{point masses}}\nonumber\\
&+ \underbrace{\sum_k h_k(\vartheta_k)(\lambda\otimes\delta_{\sigma_k}) + \sum_n\beta_n(\delta_{\tau_n}\otimes\lambda)}_{\text{line measures}}\nonumber\\
&+ \underbrace{\gamma\lambda^2}_{\text{diffuse measure}},
\label{eq:strip_representation_appendix}
\end{align}
for some measurable functions $f_n,g_m,h_k\geq 0$ on $\bbR_+^2$, a unit rate Poisson process $\{(\vartheta_k,\sigma_k)\}$ on $\bbR_+^2$, some independent $U(0,1)$ arrays $(\tau_n)$ and $(\rho_{mk})$, an independent set of random variables $\beta_n,\gamma\geq 0$, and the Lebesgue measure $\lambda$. The latter can then be chosen to be non-random if and only if $\xi$ is extreme.
\end{citedthm}

The representation theorem consists of three parts: point masses, line measures, and a diffuse measure. We select the point masses part for discrete separately exchangeable random measures. Decomposition of the entire measure $\xi$ is demonstrated in Fig~\ref{fig:slfm_appendix}.
\end{proof}

\newpage
\subsection*{Discussion on Section~\ref{subsec:crm}. Existing models as special cases of PRME model.}
Let $\xi=\sum_{n,k}f_n(\vartheta_k)\delta_{\tau_n,\sigma_k}$ be our PRME model. We focus on a specific object $n$, remove the redundant $\tau_n$, and directly work on random measures on $\Theta$. This transformation let us be on the same page of other research on completely random measures.

We have $\xi_n=\sum_{k}f_n(\vartheta_k)\delta_{\theta_k}$ be a population random measure embedding model, where $(\vartheta_k,\theta_k)$ is a Poisson process on $\bbR_+\times\Theta$ with mean measure $p(\theta)d\vartheta d\theta$. The according CRM is $\lambda=\sum_{k}\widetilde{\vartheta}_k\delta_{\theta_k}$ with Lévy measure $\nu(d\widetilde{\vartheta},d\theta)=\mu(\widetilde{\vartheta})p(\theta)d\widetilde{\vartheta}d\theta$. Assume the tail function~$T(\widetilde{\vartheta})=\nu((\widetilde{\vartheta},\infty),\Theta)$ is invertible. One can do a transformation between atoms by $(\vartheta_k,\theta_k)\rightarrow(T^{-1}(\vartheta_k),\theta_k)=(\widetilde{\vartheta}_k,\theta_k)$. The following examples are just special cases of this transformation, as we shall see.

\subsubsection*{IBP and extensions}
The Indian buffet process (IBP) take a particular form $\xi_n=\sum_{k}f\circ T^{-1}(\vartheta_k)\delta_{\theta_k}$, where $f$ are independent Bernoulli random variables with success rate $T^{-1}(\vartheta_k)$. IBP uses a particular transformation $T^{-1}(\vartheta_k)=e^{-\vartheta_k}$~\cite{Thibaux07}.~\cite{Teh09} gives a power-law extension of IBP with three parameters (3IBP) with an application in language models. However, 3IBP does not enjoy an analytical form for $T^{-1}$. But we can safely work on the CRM directly, given the generality of the existence of $T^{-1}$~\cite{Orbanz11}. One can observe that the sampling function $f\circ T^{-1}$ does not change with $n$. This is the main limitation for IBP and 3IBP. A MCMC sampling solution can be found in~\cite{Griffiths11,Teh09}.

\subsubsection*{Correlated random measures}
The key restrictions of IBP and 3IBP is that $\bbE[\xi_n(\{\theta_{k_1}\})\cdot\xi_n(\{\theta_{k_2}\})|\widetilde{\vartheta}]=\bbE[\xi_n(\{\theta_{k_1}\})|\widetilde{\vartheta}]\cdot\bbE[\xi_n(\{\theta_{k_2}\})|\widetilde{\vartheta}]$. In order to model feature correlations,~\cite{Paisley12} model $f_n(\vartheta_k)$ as exchangeable random functions. The extra randomness besides $\widetilde{\vartheta}$ can be modelled by augmenting the Poisson process $(\widetilde{\vartheta}_k,\theta_k)$ on $\bbR_+\times\Theta$ to higher dimension $(\ell_k,\widetilde{\vartheta}_k,\theta_k)$ on $\bbR^d\times\bbR_+\times\Theta$ with mean measure $\nu(d\ell,d\widetilde{\vartheta},d\theta)=p(\ell)\mu(\widetilde{\vartheta})p(\theta)d\ell d\widetilde{\vartheta}d\theta$. The discrete infinite logistic normal distribution (DILN)~\cite{Paisley12} further proposes an example $\xi_n=\sum_k Z_{nk}(\beta\widetilde{\vartheta}_k,\exp(-h_n(\ell_k)))\delta_{\sigma_k,\tau_n}$, where $h_n(\cdot)\sim\text{GP}(m(\cdot),K(\cdot,\cdot))$ and $Z_{nk}$ is a gamma distribution parameterized by its shape and scale parameters. However, DILN is restricted to use linear kernels, which is very restrictive.~\cite{Ranganath18} proposed general correlated random measures with examples for the binary, discrete, and continuous cases.

\subsection*{Section~\ref{sec:topic_model}. Detailed derivations.}
The variational objective function can be decoupled as
\begin{align}
\cL&=\sum_{k=1}^K\bbE\Big[\ln p(\ell_k)+\ln p(V_k)+\ln p(\theta_k)\Big]\nonumber\\
&+\sum_{n=1}^N\bbE\Big[\ln p(h_n)\Big] +\sum_{n=1}^N\sum_{k=1}^K\bbE\Big[\ln p(Z_{nk}|V,h_n,\ell_k)\Big]\nonumber\\
&+\sum_{n=1}^N\sum_{m=1}^{M_n}\bbE\Big[\ln p(C_n^{(m)}|Z_n)+\ln p(X_n^{(m)}|C_n^{(m)},\theta)\Big]\nonumber\\
&+\bbH\Big[q(\ell,V,\theta,h,Z,C)\Big].
\label{eq:vi_obj_batch}
\end{align}
We expand each term in Eq.~(\ref{eq:vi_obj_batch}) as follows.
\begin{align}
&\bbE[\ln p(\ell_k)]=\ln p(\widehat{\ell}_k)=-\frac{r_\ell\ln(2\pi b)}{2}-\frac{\widehat{\ell}_k^\top\widehat{\ell}_k}{2b}.\\
&\bbE[\ln p(V_k)]\!=\!\ln p(\widehat{V}_k)\!=\!\ln\alpha+(\alpha-1)\ln(1-\widehat{V}_k).\\
&\bbE[\ln p(\theta_k)]=\ln\Gamma(D\gamma_0)-D\ln\Gamma(\gamma_0)\nonumber\\
&\qquad\qquad\quad+\sum_{d=1}^D (\gamma_0-1)\bbE[\ln\theta_{kd}],\nonumber\\
&\text{where }\bbE[\ln\theta_{kd}]=\psi(\gamma_{kd})-\psi(\sum_{d=1}^D\gamma_{kd}).
\end{align}

\begin{align}
\bbE[\ln p(h_n)]=\ln p(\widehat{h}_n)=-\frac{r_h\ln(2\pi a)}{2}-\frac{\widehat{h}_n^\top\widehat{h}_n}{2a}.
\end{align}

\begin{align}
&\bbE[\ln p(Z_{nk}|V,h_n,\ell_k)]\!=\!-\ln\Gamma(\beta \widehat{p}_k)\!-\!\beta \widehat{p}_k\bbE[f(h_n,\ell_k)]\nonumber\\
&\quad+(\beta\widehat{p}_k-1)\bbE[\ln Z_{nk}]-\bbE[Z_{nk}]\bbE[\exp(-f(h_n,\ell_k))],\nonumber\\
&\text{where } \bbE[f(h_n,\ell_k)]=\mu_f(\widehat{h}_n,\widehat{\ell}_k),\nonumber\\ &\bbE[\exp(-f(h_n,\ell_k))]\!=\!\exp\Big(\!\!-\!\mu_f(\widehat{h}_n,\widehat{\ell}_k)\!+\!\frac{1}{2}\sigma^2_f(\widehat{h}_n,\widehat{\ell}_k)\!\Big),\nonumber\\
&\bbE[\ln Z_{nk}]=\ln(b_{nk})+\psi(a_{nk}),~\bbE[Z_{nk}]=a_{nk}b_{nk}.
\end{align}

\begin{align}
\bbE[\ln p(C_{nm}|Z_n)]\!=\!\sum_{k=1}^K\!\phi_{nm}(k)\bbE\Big[\!\ln\! Z_{nk}\!-\!\ln\!\sum_{k'=1}^K\!Z_{nk'}\!\Big].
\end{align}

\begin{align}
&\bbE[\ln p(X_{nm}|C_{nm},\theta)]=\sum_{k=1}^K\phi_{nm}(k)\bbE[\ln\theta_{k,X_{nm}}],\nonumber\\
&\text{where }\bbE[\ln\theta_{k,X_{nm}}]=\psi(\gamma_{k,X_{nm}})-\psi(\sum_{d=1}^D\gamma_{kd}).\\
&\bbH[q(\ell_k)]=0.\\
&\bbH[q(V_k)]=0.\\
&\bbH[q(\theta_k)]=\sum_{d=1}^D\ln\Gamma(\gamma_{kd})-\ln\Gamma(\sum_{d=1}^D\gamma_{kd})\nonumber\\
&\qquad\qquad-\sum_{d=1}^D (\gamma_{kd}-1)\bbE[\ln\theta_{kd}].\\
&\bbH[q(h_n)]=0.\\
&\bbH[q(Z_{nk})]=a_{nk}+\ln(b_{nk})+\ln\Gamma(a_{nk})\nonumber\\
&\qquad\qquad\quad+(1-a_{nk})\psi(a_{nk}).\\
&\bbH[q(C_{nm})]=-\sum_{k=1}^K \phi_{nm}(k)\ln\phi_{nm}(k).
\end{align}
Variational inference for $\ell_k,V_k$ and network parameters can be done by directly plug-in and take gradients. Updating $q(\theta_k)$ and $q(C_{nm})$ follows the general variational update rule. Updating $q(Z_{nk})$ requires lower-bounding $\cL$.
\\~\\
For $\ell$, we use gradient ascent:
\begin{align}
\nabla_\ell{\cL} &= \sum_{k=1}^K\nabla_\ell\bbE\Big[\ln p(\ell_k)\Big] \nonumber\\
&+ \sum_{n=1}^N\sum_{k=1}^K\nabla_\ell\bbE\Big[\ln p(Z_{nk}|V,h_n,\ell_k)\Big].
\label{eq:update_ell_appendix}
\end{align}
For $V$, we use gradient ascent:
\begin{align}
\nabla_V{\cL} &= \sum_{k=1}^K\nabla_V\bbE\Big[\ln p(V_k)\Big] \nonumber\\
&+ \sum_{n=1}^N\sum_{k=1}^K\nabla_V\bbE\Big[\ln p(Z_{nk}|V,h_n,\ell_k)\Big].
\label{eq:update_v_appendix}
\end{align}
For $\theta$, we have a closed-form update:
\begin{align}
\gamma_{kd}=\gamma_0+\sum_{n=1}^N\sum_{m=1}^{M_n}\phi_{nm}(k)\cdot\mathbbm{1}(X_{nm}=d)
\label{eq:update_eta_appendix}
\end{align}
For $h_n$, we update the inference network $g$: 
\begin{align}
\nabla_g\cL&=\sum_{n=1}^N\nabla_g\bbE\Big[\ln p(h_n)\Big]\nonumber\\
&+ \sum_{n=1}^N\sum_{k=1}^K\nabla_g\bbE\Big[\ln p(Z_{nk}|V,h_n,\ell_k)\Big]
\label{eq:update_h_appendix}
\end{align}
For $Z_{nk}$, we maximize a lower bound for $\cL$ similar as~\cite{Paisley12}. Related terms in $\cL$ are:
\begin{align}
\cL(q(Z_{nk}))&=(\beta\widehat{p}_k-1+\sum_{m=1}^{M_n}\phi_{nm}(k))\bbE[\ln Z_{nk}]\nonumber\\
&-\bbE[Z_{nk}]\bbE[\exp(-f(\widehat{h}_n,\widehat{\ell}_k))]\nonumber\\
&-M_n\bbE\Big[\ln\sum_{k'=1}^K Z_{nk'}\Big]+\bbH[q(Z_{nk})].
\end{align}
The term that make closed-form update intractable is $\bbE[\ln\sum_{k'=1}^K Z_{nk'}]$. We use the bound:
\begin{align}
\bbE[\ln\sum_{k'=1}^K Z_{nk'}]\leq\ln\varepsilon_n+\frac{\sum_{k'=1}^K\bbE[Z_{nk'}]-\varepsilon_n}{\varepsilon_n}.
\end{align}
This bound is correct for any $\varepsilon_n>0$, and here we precompute $\varepsilon_n=\sum_{k=1}^K\bbE[Z_{nk}]$ and treat $\varepsilon_n$ as a constant in the above equation. After Plugging-in the bound and some algebra, we solve $q(Z_{nk})$ as:
\begin{align}
a_{nk}&=\beta\widehat{p}_k+\sum_{m=1}^{M_n}\phi_{nm}(k),\nonumber\\ 1/b_{nk}&=\bbE\Big[\exp(-f(h_n,\ell_k))\Big]+\frac{M_n}{\varepsilon_n}.
\label{eq:update_z_appendix}
\end{align}
\\~\\
For decoder network $f$:
\begin{align}
\nabla_f\cL=\sum_{n=1}^N\sum_{k=1}^K\nabla_f\bbE\Big[\ln p(Z_{nk}|V,h_n,\ell_k)\Big].
\label{eq:update_f_appendix}
\end{align}
\\~\\
For $C_{nm}$, we have a closed-form update:
\begin{align}
\phi_{nm}(k)\propto\exp\Big(\bbE[\ln\theta_{k,X_{nm}}]+\bbE[\ln Z_{nk}]\Big).
\label{eq:update_c_appendix}
\end{align}

\end{document}